%% file: main.tex
\documentclass{article}

\usepackage{mypackage}
\numberwithin{equation}{section}

\makeatletter
\long\def\@makecaption#1#2{%
  \normalsize%% add this line
  \vskip\abovecaptionskip
  \sbox\@tempboxa{#1: #2}%
  \ifdim \wd\@tempboxa >\hsize
    #1: #2\par
  \else
    \global \@minipagefalse
    \hb@xt@\hsize{\hfil\box\@tempboxa\hfil}%
  \fi
  \vskip\belowcaptionskip}
\makeatother

\title{Asymptotic Behavior of Bayesian Generalization Error in Multinomial Mixtures}
\author[1]{Takumi Watanabe\thanks{watanabe.t.bv@m.titech.ac.jp}} 
\author[2]{Sumio Watanabe\thanks{swatanab@c.titech.ac.jp}}
\affil[1]{{\normalsize Department of Mathematical and Computing Science, Tokyo Institute of Technology, 2-12-1, Oookayama, Meguro-ku, Tokyo, 152-8552, Japan}}
\affil[2]{{\normalsize Department of Mathematical and Computing Science, Tokyo Institute of Technology, 2-12-1, Oookayama, Meguro-ku, Tokyo, 152-8552, Japan}}
\date{}

\begin{document}

\maketitle

\begin{abstract}
  \input{abstruct}
\end{abstract}

\tableofcontents

\section{Introduction}
  \label{sec:Introduction}
  \input{introduction}

\section{Bayes estimation}
  \label{sec:Bayes estimation}
  \input{Bayes_estimation}

\section{Multinomial Mixtures}
  \label{sec:Multinomial Mixtures}
  \input{multinomial}

\section{Previous Studies}
  \label{sec:Previous Studies}
  \input{previous_studies}

\section{Main Theorem}
  \label{sec:Main Theorem}
  \input{main_theorem}

\section{Proof of the Main Theorem}
  \label{sec:Proof of the Main Theorem}

  \subsection{properties of the RLCTs and the multiplicities}
    \label{subsec:properties_RLCT}
    \input{properties_RLCT}
  
  \subsection{The restriction on the parameter set of general multinomial mixtures}
    \label{subsec:The restriction on the parameter set of general multinomial mixtures}
    \input{restriction}

  \subsection{properties of the general components}
    \label{subsec:properties_general_components}
    \input{properties_general_components}

  \subsection{properties of the 2 components}
    \label{subsec:properties_two_components}
    \input{properties_two_components}

  \subsection{Proof of the main theorem}
    \label{subsec:Proof of the main theorem}
    \input{main_proof}

\section{Phase transition due to prior distribution hyperparameters}
  \label{sec:Phase transition due to prior distribution hyperparameters}
  \input{phase_transition}

\section{Conclusions}
  \label{sec:Conclusions}
  \input{conclusions}

\clearpage
\bibliographystyle{junsrt}
\bibliography{references}
\nocite{*}     

\end{document}

%% file: abstruct.tex
Multinomial mixtures are widely used in the information engineering field, however, their mathematical properties are not yet clarified because they are singular learning models.
In fact, the models are non-identifiable and their Fisher information matrices are not positive definite. In recent years, the mathematical foundation of singular statistical models are clarified by using algebraic geometric methods. In this paper, we clarify the real log canonical thresholds and multiplicities of the multinomial mixtures
and elucidate their asymptotic behaviors of generalization error and free energy.\\
\textit{\textbf{Keywords}}: generalization error, free energy, multinomial mixtures, real log canonical thresholds

%% file: introduction.tex
A finite mixture model is a probability distribution defined by a linear superposition of a finite number of distributions. 
Its examples, A normal mixture, a Poisson mixture, and a multinomial mixture have been applied in many research areas.
In this paper, we mainly study a multinomial mixture which provides a richer class of statistic models than the single multinomial distribution. The multinomial mixture has been applied to document clustering \cite{watanabe2002Prototyping}, anomaly detection in medical data \cite{masada2007clusterning}, and image clustering \cite{masada2007clustering}. 
In spite of a wide range of its applications, their mathematical property of generalization performance is not yet clarified. One of the mathematical difficulties is caused by the fact that they are not identifiable.
If the map from the set of parameters to a probability distribution, then a statistical model is called identifiable. If otherwise, it is called non-identifiable \cite{yamazaki2002resolution}. 
These models are classified into singular models. 
If a probability model is singular, the posterior distribution cannot be approximated by any normal distribution, and classical model criteria of regular statistical models such as AIC \cite{akaike1974new}, BIC \cite{schwarz1978estimating}, or  DIC \cite{spiegelhalter2002bayesian} cannot be applied to estimate the generalization losses of singular models.

Recently, in order to establish the mathematical foundation of Bayesian inference of singular models, 
Watanabe derived the asymptotic behavior of their generalization error $G_n$ and the free energy $F_n$ \cite{watanabe2001algebraicA}. 
There exist both a real and positive number $\lambda$ and an integer $m$, such that 
their asymptotic behaviors of $G_n$ and $F_ n$ are respectively given by 
\begin{align}
  \label{eq:asymptotic Fn_intro}
  \bbE[F_n] &= nS + \lambda \log n - (m-1) \log \log n + O(1),\\
  \label{eq:asymptotic Gn_intro}
  \bbE[G_n] &= \frac{\lambda}{n} - \frac{m-1}{n \log n } + o\left( \frac{1}{n \log n} \right),
\end{align}
where $\lambda$ is called a real log canonical threshold (RLCT), $m$ is called a multiplicity, and $\bbE[\cdot]$ shows the expectation value over all datasets.
If a learning model is identifiable and regular, $\lambda = d/n, \ m = 1$, where $d$ is the dimension of the parameter space \cite{schwarz1978estimating}. 
If it is singular, $\lambda$ and $m$ depend on the true distribution, the model, and the prior. 
In singular case, it was shown by \cite{watanabe2018mathematical} that both RLCT and multiplicity can be found by using desingularization theorem in algebraic geometry. 
In general, the parameter set $K(w) = 0$ contains complicated singularities, hence it is difficult to find the resolution map, however, both RLCTs and multiplicities have been clarified in several statistical models and learning machines. 
Examples of the models in which the RLCTs are found include normal mixtures \cite{aoyagi2005stochastic}, Poisson mixtures \cite{sato2019bayesian}, Bernoulli mixtures \cite{yamazaki2013comparing}, rank regression \cite{aoyagi2005stochastic}, Latent Dirichlet Allocation (LDA) \cite{hayashi2021exact}, and so on. In addition, the RLCTs are used as an analysis of the exchange rate of the replica exchange method \cite{nagata2008asymptotic}, which is one of the Markov chain Monte Carlo methods. 
Moreover, in recent years, the information criterion $\sBIC$, which uses RLCTs in calculation, has also been proposed  \cite{drton2017bayesian}. 

In this paper, we clarify the RLCT of the multinomial mixtures and derive the asymptotic behaviors of the generalization error and the free energy. Our analysis also shows the effect of hyperparameter for the case when Dirichlet distribution is employed as a prior. 
We begin in section \ref{sec:Bayes estimation} with the introduction of the framework of Bayesian inference. In Section \ref{sec:Multinomial Mixtures} we explain multinomial mixtures, and in section \ref{sec:Previous Studies} we introduce previous studies about the RLCTs and multiplicities of multinomial mixtures. 
In Section \ref{sec:Main Theorem} we claim the main theorem of this paper. In Section \ref{sec:Proof of the Main Theorem} we prove the theorem. In Section \ref{sec:Phase transition due to prior distribution hyperparameters} we discuss the phase transition due to the hyperparameters. And in Section \ref{sec:Conclusions} we conclude this paper.

%% file: Bayes_estimation.tex
In this section, we introduce the framework of Bayesian inference. Let $q(x)$ be a \textbf{true probability distribution} 
and let $X^n = (X_1, \dots, X_n)$ be a set of training data generated from $q(x)$ independently and identically.
Let $p(x | w)$ be a \textbf{probability model}, where $w \in W \subset \bbR^d$ is a parameter, and the $W$ is a parameter space. 
The \textbf{prior probability distribution} $\varphi(w)$ is a function on $W$. A \textbf{posterior distribution} $p(w|X^n)$ is defined by
\begin{align}
  \label{eq:def:posteriori distribution}
  p(w|X^n) = \frac{1}{Z_n} \varphi(w) \prodin p(X_i|w),
\end{align}
where $Z_n$ is the normalizing constant:
\begin{align}
  \label{eq:def:Zn}
  Z_n = \intw{\varphi(w) \prodin p(X_i|w)}.
\end{align}

The constant $Z_n$ is called a \textbf{marginal likelihood function}. The \textbf{free energy} $F_n$ is defined as the minus log marginal likelihood function:
\begin{align}
  \label{eq:def:Fn}
  F_n = - \log Z_n.
\end{align}

A \textbf{predictive distribution} $p(x|X^n)$ is given by
\begin{align}
  \label{eq:def:predictive distribution}
  p(x|X^n) = \intw{p(x|w)p(w|X^n)}.
\end{align}

A \textbf{generalization error} $G_n$ is the Kullback-Leibler divergence from the true distribution $q(x)$ to the predictive distribution $p(x|X^n)$:
\begin{align}
  \label{eq:def:Gn}
  G_n = \intx{q(x) \log \frac{q(x)}{p(x|X^n)}}.
\end{align}

The generalization error is a measure of how the predictive distribution
$p(x|X^n)$ is different from the true distribution $q(x)$.% and it depends on $q(x)$. However, in most cases we do not know the distribution $q(x)$. 

For an arbitrary function $f: x^n \mapsto f(x^n)$, the expectation value of $f(X^n)$ over all sets of training samples is denoted by $\bbE[\cdot]$, that is, 
\begin{align}
  \label{eq:def:expectation on sample}
  \bbE[f(X^n)] = \int \dots \int f(x_1, \dots, x_n) \prodin q(x_i) \mathrm{d} x_i.
\end{align}

Let the \textbf{mean error function} $K(w)$ be the Kullback-Leibler divergence from the true distribution to the probability model:
\begin{align}
  \label{eq:def:K(w)}
  K(w) = \intx{q(x) \log \frac{q(x)}{p(x|w)}}.
\end{align}

An \textbf{entropy} of the true distribution $S$ and an \textbf{empirical entropy} $S_n$ are defined respectively by
\begin{align}
  \label{eq:def:S} 
  S &= - \intx{q(x) \log q(x)},\\
  \label{eq:def:Sn}
  S_n &= - \frac{1}{n} \sumin \log q(X_i).
\end{align}

It is known that the following relationship holds between the free energy and the generalization error  \cite{watanabe2009algebraic}:
\begin{align}
  \label{eq:relation between Gn and Fn}
  \bbE[G_n] &= \bbE[F_{n+1}] - \bbE[F_{n}]-S.
\end{align}
The relation \eqref{eq:relation between Gn and Fn} is important because in the most case, we do not know the true distribution $q(x)$, whereas the free energy can be calculated by using the prior $\varphi(w)$, the probability model $p(x|w)$, and a sample $X^n$.

Let $\mathrm{Re}(z)$ be the real part of a complex number $z$. Define the \textbf{zeta function} in the statistical learning theory as 
\begin{align}
  \label{eq:def:zeta}
  \zeta(z) = \intw{K(w)^z \varphi(w)}.
\end{align}
If $K(w)\geq 0 $ is an analytic function of $w$, then the function $\zeta(z)$ is holomorphic in the region $\mathrm{Re}(z) > 0$, and it can be analytically continued to the unique meromorphic function onto the entire complex plane. 
Moreover, it is known that all poles are real and negative numbers.

In the following, assume that the mean error function $K(w)$ is analytical and the true distribution is feasible with a probabilistic model. Here, the true distribution $q(x)$ is said to be feasible with the probability model $p(x|w)$ if there is a parameter $w^* \in W$ such that $q(x) = p(x|w^*) $ holds for all $x$. 
Assume that the maximum pole of the zeta function $\zeta(z)$ is $-\lambda $ and its order is $m$. 
By applying the Hironaka resolution theorem in the algebraic geometrical method, the asymptotic behaviors of free energy and generalization errors can be expressed as follows \cite{watanabe2001algebraicA} \cite{watanabe2001algebraicB}:
\begin{align}
  \label{eq:asymptotic Fn}
  \bbE[F_n] &= nS + \lambda \log n - (m-1) \log \log n + O(1),\\
  \label{eq:asymptotic Gn}
  \bbE[G_n] &= \frac{\lambda}{n} - \frac{m-1}{n \log n } + o\left( \frac{1}{n \log n} \right).
\end{align}

The constant $\lambda$ and  $m$ are called \textbf{real log canonical threshold (RLCT)} and a \textbf{multiplicity}
respectively.

%% file: multinomial.tex
\subsection{Multinomial Distribution}
\label{subsec:Multinomial Distribution}
Let $\bbZ_{\geq 0}$ be the set of all non-negative integers, $\bbR_{\geq 0}$ be the set of all non-negative real numbers.
Let constants $L, M$ be two or more natural numbers, and the set $D$ is defined by
\begin{align}
  \label{eq:def:D}
  D = \left\{ x = (x_1, \dots, x_L) \in \left( \bbZ_{\geq 0} \right)^L \ : \ \sumlL x_\ell = M \right\}.
\end{align}
The vectors $ b \in \bbR^L $ belong to the set $B$:
\begin{align}
  \label{eq:def:B}
  B = \left\{ b = (b_1, \dots, b_L) \in \left(\bbR_{\geq 0}\right)^L \ : \ \sumlL b_\ell = 1, \ 0 \leq b_\ell \leq 1 \right\}.
\end{align}
The probability distribution of $x \in D$ determined by the vector $b$:
\begin{align*}
  \rmMul_L(x|b) = \frac{M!}{\prodlL x_\ell !} \prodlL (b_\ell)^{x_\ell}
\end{align*}
is called the \textbf{multinomial distribution}.
Here, it is defined as $0^0 = 1, \ 0! = 1$. 
The constant $M$ represents the number of independent trials of the multinomial distribution, and the parameter $b = (b_1, \dots, b_L)$ represents the corresponding probability. 
The multinomial distribution is a generalization of several discrete distribution. 
If $M = 1$ and $L = 2$, the multinomial distribution is called the \textbf{Bernoulli distribution}. If $M = 1$ and $L \geq 2$, it is called the \textbf{categorical distribution}. If $M \geq 2$ and $L = 2$, it is called the \textbf{binomial distribution}.

\subsection{Multinomial Mixtures}
\label{subsec:Multinomial Mixtures}
Let $H$ be a finite natural number greater than or equal to 2. The parameter set $W$ is defined by
\begin{align}
  \label{eq:def:W}
  W = \left\{(a, b) \ : \ \sumhH a_h = 1, \ 0 \leq a_h \leq 1, \ b_h = (b_{h1}, \dots, b_{hL}) \in B \ (\forall h \in [H]) \right\},
\end{align}
where $[H]$ means the set $\{ h \in \bbZ : 1 \leq h \leq H \}$.

The probability distribution on $x \in D$ determined by the parameter $w = (a, b) \ \in W$
\begin{align}
  \label{eq:def:multinomial mixture}
  p(x | w) = \sumhH a_h \rmMul_L(x|b_h)
\end{align}
is called a \textbf{multinomial mixture}. 
Here $H$ represents the number of components. 
The $H$ dimensional parameter $a = (a_1, \dots, a_H)$ represents a mixing ratio. 
The parameter $a$ is assumed that $a_h$ is non-negative and $\displaystyle \sumhH a_h= 1$. 
Then $a_h$ represents the weight of the $h$-th component distribution. 
The higher mixing ratio $a_h$ means the stronger effect of $h$-th component.

%% file: previous_studies.tex
In this section, we introduce several previous studies on the log canonical thresholds of multinomial mixtures. When the probability model is a binomial mixture, the upper bound of the RLCT in the case of general components and the exact value of one in special cases have been clarified \cite{yamazaki2013comparing}.

%%%%%%%%%%%Theorem 4.1
\begin{theorem}[\textbf{the RLCT and multiplicity of binomial mixtures} \cite{yamazaki2013comparing}]
Let $x = (y_1, \dots, y_M) \in \{ 0, 1 \}^M$ be an $M$ dimensional binary vector and $p(x|w)$ be a binomial mixture, 
\begin{align}
  \label{eq:preA:p(x|w)}
  p(x|w) = \sumhH a_h \prodmM p^{ym}_{hm} (1 - p_{hm})^{1 - y_m}.
\end{align}
It is assumed that the true distribution $q(x)$ is a binomial mixture, 
\begin{align}
  \label{eq:preA:q(x)}  
  q(x) = \sumhHz a^*_h \prodmM p^{*ym}_{hm} (1 - p^*_{hm})^{1-y_m}.
\end{align} 
Let the prior distribution $\varphi(w)$ be 
\begin{align}
  \label{eq:preA:varphi(w)}
  \varphi(w; \eta) = \varphi_0(a; \alpha) \prodhH \prodmM \varphi_1(p_{hm}; \beta),
\end{align}
where $\eta = \{ \alpha, \beta \}$ is a set of hyperparameters, $\varphi_0(a; \alpha)$ is a prior distribution of the mixing ratio $a$ with $\alpha>0$ as a hyperparameter (Dirichlet distribution), 
\begin{align}
  \label{eq:preA:varphi(a)}
  \varphi_0(a; \alpha) &= \frac{\Gamma(H\alpha)}{\Gamma(\alpha)^H} 
\left(\prodhHm a_h^{\alpha-1} \right) \left( 1 - \sumiHm a_i \right)^{\alpha-1}.
\end{align}  
and  $\varphi_1(p_{hm}; \beta)$ ($\beta>0$) is a beta distribution for each $h \in [H], m \in [M]$,
\begin{align}
  \label{eq:preA:varphi(phm)}
  \varphi_1(p_{hm}; \beta) &= \frac{\Gamma(2\beta)}{\Gamma(\beta)} p_{hm}^{\beta-1}(1 - p_{hm})^{\beta-1}.
\end{align}  
We refer to them as deterministic, where $p^*_{hm}$ is one or zero. 
Let $H_1$ and $H_2$ be the numbers of probabilistic and deterministic components, respectively, where $H = H_1 + H_2$.
Under the above conditions, the asymptotic behavior of the free energy $F_n$ is expressed as follows:
\begin{align}
  \label{eq:preA:Fn}
  F_n \leq nS+\mu \log n - (m_{\mu} - 1) \log \log n + o \left( \log \log n \right),
\end{align}
where $S$ is the entropy of the true distribution, and $\mu$ and $m_{\mu}$ are defined as follows. 
For $M \geq 3$,
\begin{align}
  \label{eq:preA:M3mu}
  \mu &= \frac{H_0 - 1 + H_1 M + H_2 M\beta}{2} + \frac{H - H_0}{2} \min\left\{ \alpha, \ \frac{M}{2}, \ \frac{\beta M}{2} \right\},\\
  \label{eq:preA:M3m}
  m_{\mu} &=
  \begin{cases}
    2 & (\alpha = \min \{ M/2, \ \beta M / 2 \})\\
    1 & (\mathrm{otherwise})
  \end{cases}.
\end{align}
For $M = 2$,
\begin{align}
  \label{eq:preA:M2mu}
  \mu &= \frac{K_0 - 1 + H_1 M + H_2 M\beta}{2} + \frac{H - H_0}{2} \min\left\{ \alpha, \ 1, \ \beta \right\},\\
  \label{eq:preA:M2m}
  m_{\mu} &=
  \begin{cases}
    3 & (\alpha = \min \{ 1, \ \beta \})\\
    2 & (\alpha > \min \{ 1, \ \beta \})\\
    1 & (\mathrm{otherwise})
  \end{cases}. 
\end{align}
Furthermore, if $H = H_0 + 1$, that is, when the number of components in the probability model is one greater than that in the true distribution, $ \mu $ equals the exact value of the RLCT $\lambda$, and $ m_\mu$ also equals the multiplicity $m$.
\end{theorem}

Matsuda analyzed the RLCT of trinomial mixtures with two components. 
The exact value of the RLCTs was elucidated when the true distribution is multinomial distribution and the probability model is trinomial mixtures with two components, that is, in the case of $L = 3, H = 2$.

%%%%%%%% Theorem 4.2 
\begin{theorem}[\textbf{the RLCT and multiplicity of trinomial mixtures with two components} \cite{matsuda2008weighted}]
\label{theorem:matsuda}
Let the probability model $p(x|w)$ be a trinomial mixture with two components:
\begin{align}
  \label{eq:preB:p(x|w)}
  p(x|w) = a\rmMul_3(x|b_1) + (1-a)\rmMul_3(x|b_2), \ (a, b_1, b_2) \in W.
\end{align}
Here, $\rmMul_3(x|b)$ means a probability mass function of trinomial distribution with $b = (b_1, b_2, b_3)$ as parameters. 
Also, let the true distribution $q(x)$ be a trinomial distribution:
\begin{align}
  \label{eq:preB:q(x)}
  q(x) = \rmMul_3(x|b^*).
\end{align}
Also, assume that the prior distribution $\varphi(w)$ is greater than $0$ and bounded above the parameter set $W$, and the true distribution parameter $b^* = (b_1^*, b_2^*, b_3^*)$ satisfies:
\begin{align}
  \label{eq:preA:condition of b*}
  b_1^* b_2^* b_3^* \neq 0.
\end{align}
Under these conditions, The RLCT is as follows:
\begin{align}
  \label{eq:preA:RLCT}
  \lambda = \frac{3}{2}.
\end{align}
\end{theorem}

Matsuda clarified the RLCT of a trinomial mixture with two components, which is in the case of $H = 2, H^* = 1, L = 3$, by using an algebraic geometry algorithm called weighted blow-up.

%% file: main_theorem.tex
In this section, we show the main result of this paper, which is a generalization of Theorem \ref{theorem:matsuda}. 
We clarify the RLCT and the multiplicity of general multinomial mixtures with two components. 
Furthermore, we also consider the case where the Dirichlet distribution is adopted as the prior distribution of the mixture ratio.

\begin{theorem}[\textbf{Main Theorem}]
\label{maintheorem}
Let the probability model $p(x|w)$ be a multinomial mixture with two components:
\begin{align}
  \label{eq:maintheorem:p(x|w)}
  p(x|w) = a\rmMul_L(x|b_1) + (1-a)\rmMul_L(x|b_2), \ (a, b_1, b_2) \in W.
\end{align}
Also, let the true distribution be a multinomial distribution:
\begin{align}
  \label{eq:maintheorem:q(x)}
  q(x) = \rmMul_L(x|b^*).
\end{align}
Also, assume that the prior distribution of the parameter $b$ is greater than $0$ and bounded in the set $W$, and that the true distribution parameter $ b^* = (b_1^*, \dots, b_L^*) $ satisfies:
\begin{align}
  \label{eq:maintheorem:condition of b*}
  \prodlL b_\ell^* \neq 0.
\end{align}
The prior distribution of the mixing ratio $a$ in two cases as follows respectively:
\begin{enumerate}
  \item \label{maintheorem:enm1}
  If the prior distribution $\varphi(a)$ of mixture ratio $a$ is greater than $0$ and bounded, the RLCT $\lambda$ and the multiplicity $m$ are given by
  \begin{align} 
  \label{eq:maintheorem:caseA:RLCT}
  \lambda &= \frac{L-1}{2} + \min \left(\frac{1}{2}, \ \frac{L-1}{4} \right), \\
  \label{eq:maintheorem:caseA:m}
  m &=
  \begin{cases}
    2 & (L = 3)\\
    1 & (\mathrm{otherwise})
  \end{cases}.
  \end{align}
  \item \label{maintheorem:enm2}
  If the prior distribution $\varphi(a)$ of mixture ratio $a$ is the Dirichlet distribution with $\alpha \ ( > 0)$ as a hyperparameter:
  \begin{align}
    \label{eq:maintheorem:caseB:varphi(a)}
    \varphi(a; \alpha) = \frac{\Gamma(2\alpha)}{\Gamma(\alpha)^2} a^{\alpha -1 }(1 - a)^{\alpha - 1},
  \end{align}
  the RLCT $\lambda$ and the multiplicity $m$ are given by
  \begin{align}
    \label{eq:maintheorem:caseB:RLCT}
    \lambda &= \frac{L-1}{2} + \min \left(\frac{\alpha}{2}, \ \frac{L-1}{4} \right), \\
    \label{eq:maintheorem:caseB:m}
    m &=
  \begin{cases}
    2 & (\alpha = \frac{L-1}{2})\\
    1 & (\mathrm{otherwise})
  \end{cases}.    
  \end{align}  
\end{enumerate}
\end{theorem}

%% file: properties_RLCT.tex
To prove the theorem \ref{maintheorem}, we introduce notations, and explain some properties of the RLCTs and the multiplicities. Since the RLCT $\lambda$ and the multiplicity $m$ are determined by the mean error function $K(w)$ and the prior distribution $\varphi(w)$, they are expressed as $\lambda(K, \varphi), m(K, \varphi)$, or $\lambda(K), m(K)$, respectively. If the maximum pole and their order of the two zeta function $\zeta_1(z), \zeta_2(z)$:
\begin{align}
  \label{eq:two zeta functions1}
  \zeta_1(z) &= \intw{K(w)^z \varphi(w)},\\
  \label{eq:two zeta functions2}
  \zeta_2(z) &= \intw{K'(w)^z \varphi(w)},
\end{align}
are equal, they are written as 
\begin{align}
  \label{eq:K(w)equality}
  K(w) \sim K'(w),
\end{align}
or $\lambda(K, \varphi) = \lambda(K', \varphi'), \ m(K, \varphi) = m(K', \varphi')$.

The following properties hold for the RLCTs and multiplicities \cite{watanabe2018mathematical}\cite{matsuda2008weighted}.
%%%%%%%%%%%%%%% Lemma 5.1
\begin{lemma}
\label{lem:zeta}
Let $K(w)$ be the mean error function and let $\varphi(w)$ be the prior function.
\begin{enumerate}
  \item \label{lem:enm:inequality} 
  If there are function $K'(w)$ and constants $c, c' > 0$ exist, and
  \begin{align}
    \label{lem:eq:inequality}
    cK'(w) \leq K(w) \leq c'K'(w)
  \end{align}
  holds for any $w \in W$, then $K(w) \sim K'(w)$.
  \item \label{lem:sum}
  If $w = (w_1, w_2), \ K(w) = K_1(w_1) + K_2(w_2), \ \varphi(w) = \varphi(w_1)\varphi(w_2)$, the following holds
  \begin{align}
    \label{lem:eq:sum:lambda}
    \lambda(K, \varphi) &= \lambda(K_1, \varphi_1) + \lambda(K_2, \varphi_2),\\
    \label{lem:eq:sum:sum}
    m(K, \varphi) &= m(K_1, \varphi_1) + m(K_2, \varphi_2) - 1.
  \end{align}
  \item \label{lem:product}
  If $w=(w_1, w_2),\ K(w) = K_1(w_1)K_2(w_2),\ \varphi(w) = \varphi_1(w_1) \varphi_2(w_2)$, then
  \begin{align}
    \label{lem:eq:product:lambda}
    \lambda(K, \varphi) &= \min \Bigl(\lambda(K_1, \varphi_1),\ \lambda(K_2, \varphi_2)  \Bigr),\\
    \label{lem:eq:product:m}
    m(K, \varphi) &= 
    \begin{cases}
      m(K_1, \varphi_1) & (\lambda(K_1, \varphi_1) < \lambda(K_2, \varphi_2)) \\
      m(K_1, \varphi_1) + m(K_2, \varphi_2) & (\lambda(K_1, \varphi_1) = \lambda(K_2, \varphi_2))\\
      m(K_2, \varphi_2) & (\lambda(K_1, \varphi_1) > \lambda(K_2, \varphi_2))
    \end{cases}.
  \end{align}
  \item \label{lem:ideal}
  Let $I, J$ be natural numbers, and let $\{f_i(w)\}_{i = 1}^{I}, \{g_j(w)\}_{j = 1}^{J} $ be the sets of analytic functions. If the ideal generated from $\{f_i(w)\}_{i = 1}^{I}$ and the ideal generated from $\{g_j(w)\}_{j = 1}^{J}$ are equal and
  \begin{align}
    \label{lem:eq:ideal:K(w)}
    K_1(w) = \sum_{i=1}^{I} f_{i}(w)^2, \ K_2(w) = \sum_{j=1}^{J} g_{j} (w)^2,
  \end{align}
  then $K_1(w) \sim K_2(w)$.
  \item \label{lem:bounded}
  For any bounded function $F(w), G(w), H(w)$ on a compact set,
  \begin{align}
    \label{lem:eq:bounded}
    H(w)^2 + \left( F(w) + H(w)\ G(w) \right)^2 \sim H(w)^2 + F(w)^2.
  \end{align}
  \item \label{lem:p(x|w)-q(x)}
  Let $K'(w)$ be the following function:
  \begin{align}
    \label{lem:eq:p(x|w)-q(x):K'(w)}
    K'(w) = \sum_{x \in D} \left( p(x|w) - q(x) \right)^2,
  \end{align} 
  then $K(w) \sim K'(w)$.
\end{enumerate}
\end{lemma}

%% file: restriction.tex
To prove the theorem \ref{maintheorem}, we prepare some lemmas. As mentioned in section \ref{sec:Bayes estimation}, the zeta function $\zeta(z)$ is determined by the prior distribution $\varphi(w)$ and the mean error function $K(w)$, and the mean error function is defined by the KL information between true distribution and the probability model. Thus, the mean error function $K(w)$ is 

\begin{align}
  \label{eq:mean_error_function}
  K(w) = \sumxD q(x) \log \frac{q(x)}{p(x|w)}.
\end{align}

However, in the case of mixed multinomial distribution, some problems arise when considering the mean error function $K(w)$ for the entire parameter set W. 
When the probability model is a mixed multinomial distribution, $p(x|w) = 0 $ for some $w \in W$ and some $x \in D $.
Since the true distribution $q(x)$ is not 0 by assumption, on the points $w$ such that $p(x|w) = 0$, the values $q(x) / p(x|w)$ is not finite and the mean error function $K(w)$ diverges. 
Thus, the results of the asymptotic behavior of their generalization error in the reference \cite{watanabe2001algebraicA} cannot be applied directly. 
To solve this problem, we prove that even if the original parameter set $W$ is restricted to the set $W_1$, which is the parameter set such that $p(x|w) > 0$, the asymptotic behavior of the generalization error does not change.

\begin{lemma}[the restriction on the parameter set]
  \label{lem:the restriction on the parameter set}
  Let $W$ be a parameter set of the multinomial mixtures with component H. Let probability model be the multinomial mixtures with $H$ components :
  \begin{align}
    p(x|w) &= \sumhH a_h \rmMul_L(x|b_h) \ (w \in W).
  \end{align}
  Let $q(x)$ be the multinomial mixtures with $H^*$ components ($H \geq H^*$):
  \begin{align}
    q(x) &= \sumhHs a_h^* \rmMul_L(x|b_h^*).
  \end{align}
  Here, for any $h = 1, \dots, H$,
  \begin{align}
    \sumlL b_{h\ell} \neq 0.
  \end{align}
  Fix $0 < \varepsilon < 1$ as a sufficiently small number, and let $W_1$ be the subset of $W$ such that $p(x|w) > \varepsilon $ for all $x \in D$, and let $W_2$ be the complement of $W_1$ (i.e. $W_2 = W \setminus W_1$). Let $\lambda$ be the minus maximum pole and $m$ be its order of the zeta function whose integration range is restricted to $W_1$:
  \begin{align}
    \zeta(z) &= \int_{W_1} K(w)^z \varepsilon \mathrm{d}w.
  \end{align}
  Then, the asymptotic behavior of the generalization error is expressed by the following equation:
  \begin{align}
    \bbE[G_n] = \frac{\lambda}{n} - \frac{m-1}{n\log n} + o\left( \frac{1}{n\log n} \right)
  \end{align}
\end{lemma}

Lemma \ref{lem:the restriction on the parameter set} means that the asymptotic behavior of the generalization error of the multinomial mixtures can be analyzed by finding the maximum pole of the zeta function whose integration range is restricted to $W_1$. Lemma \ref{lem:the restriction on the parameter set} will be explained in the section \ref{subsubsec:the proof of Lemma}.

\subsubsection{the proof of Lemma \ref{lem:the restriction on the parameter set}}
\label{subsubsec:the proof of Lemma}
We will prove the Lemma \ref{lem:the restriction on the parameter set}. By the definition of the generalization error,
\begin{align}
  \bbE[G_n] &= \bbE\left[ \log \frac{q(X)}{p(X|X^n)} \right].
\end{align}
Let $X_{n+1}$ be written as $X$. By the definitions of the prediction distribution and posterior distribution, and the assumption $q(x) > 0$,
\begin{align}
   \frac{q(X_{n+1})}{p(X_{n+1}|X^n)} 
   &= \cfrac{q(X_{n+1})}{\displaystyle \intw{p(X_{n+1}|w)p(w|X^n)}}\\
   &= \frac{q(X_{n+1})}{\displaystyle \intw{\cfrac{1}{Z_n}  \varphi(w)p(X_{n+1}|w)\prodin p(X_i|w)}}\\
   &= \frac{\displaystyle q(X_{n+1}) \intw{\displaystyle \frac{1}{Z_n} \varphi(w)\prodin p(X_i|w)}}{\displaystyle \intw{\varphi(w)\prod_{i=1}^{n+1} p(X_i|w)}}\\
   &= \frac{\displaystyle  \intw{\varphi(w)\prodin \frac{p(X_i|w)}{q(X_i)}}}{\displaystyle \intw{\varphi(w)\prod_{i=1}^{n+1} \frac{p(X_i|w)}{q(X_i)}}}\\
   \label{eq:the relation between G and Z}
   &= \frac{Z(X^{n})}{Z(X^{n+1})},
\end{align}
where, $Z(X^n)$ is a quantity expressed by the following equation:
\begin{align}
  Z(X^n) = \intw{\varphi(w)\prodin \frac{p(X_i|w)}{q(X_i)}}.
\end{align}
By the Eq.\eqref{eq:the relation between G and Z},
\begin{align}
  \bbE[G_n] = \bbE\left[ \log \frac{Z(X^{n})}{Z(X^{n+1})} \right].
\end{align}
Here, fix $0 < \varepsilon < 1$ as a sufficiently small number, and let $W_1$ be the subset of $W$ that is $p(x|w) > \varepsilon$ in all $x \in D$, and let $W_2$ be the subset that is not. The integral for the parameter w is divided into two integration sets $W_1$ and $W_2$, and $Z_1(X^n)$ and $Z_2(X^n)$ are defined as follows:
\begin{align}
  Z_1(X^n) &= \intwone{\varphi(w)\prodin \frac{p(X_i|w)}{q(X_i)}},\\
  Z_2(X^n) &= \intwtwo{\varphi(w)\prodin \frac{p(X_i|w)}{q(X_i)}}.
\end{align}
Then, the relation $Z(X^n) = Z_1(X^n) + Z_2(X^n)$ holds, and
\begin{align}
  \bbE[G_n] = \bbE\left[ \log \frac{Z_1(X^{n}) + Z_2(X^{n})}{Z_1(X^{n+1}) + Z_2(X^{n+1})} \right].
\end{align}
Since $p(x|w) > 0$ on $W_1$, \cite{watanabe2018mathematical} can be applied, for suffiently large $n$, the following asymptotic behavior of $Z_1(X^n)$ holds:
\begin{align}
  Z_1(X^n) = \frac{(\log n )^{m-1}}{n^{\lambda}}Z_0(X^n),
\end{align}
where $Z_0(X^n)$ is the random variable of $X^n$ that holds
\begin{align}
  \bbE \left[ \log \frac{Z_0(X^{n+1})}{Z_0(X^n)} \right] = o\left( \frac{1}{n\log n} \right).
\end{align}
Here, we introduce the following Lemma.

\begin{lemma}[]
\label{Lemma of Theta}
There is a random variable $\Theta(X^n)$ that is $1$ only for certain events related to $X^n$ and $0$ for others, and the following equation holds:
 \begin{align}
    \label{eq:Theta1}
    \bbE\left[ \Theta(X^{n+1}) \log \frac{Z_1(X^{n}) + Z_2(X^{n})}{Z_1(X^{n+1}) + Z_2(X^{n+1})}  \right] &= O\left( \exp(-n) \right),\\
    \label{eq:Theta2}
    \bbE\left[ \left( 1 - \Theta(X^{n+1})\right) \log \frac{Z_1(X^{n}) + Z_2(X^{n})}{Z_1(X^{n+1}) + Z_2(X^{n+1})}  \right] &= \frac{\lambda}{n} - \frac{m-1}{n\log n} +  o\left( \frac{1}{n \log n} \right).
  \end{align}
\end{lemma}
If we can show the lemma \ref{Lemma of Theta}, the following holds:
\begin{align}
  \begin{split}
    \bbE[G_n] &= \bbE\left[ \Theta(X^{n+1}) \log \frac{Z_1(X^{n}) + Z_2(X^{n})}{Z_1(X^{n+1}) + Z_2(X^{n+1})}  \right]\\
    &\quad + \bbE\left[ \left( 1 - \Theta(X^{n+1})\right) \log \frac{Z_1(X^{n}) + Z_2(X^{n})}{Z_1(X^{n+1}) + Z_2(X^{n+1})}  \right]\\
  \end{split}\\
    &= \frac{\lambda}{n} - \frac{m-1}{n\log n } + o\left( \frac{1}{n \log n} \right),
\end{align}
and Lemma \ref{lem:the restriction on the parameter set} proof is completed. To prove lemma \ref{Lemma of Theta}, we prepare the section \ref{Sanov's theorem}.

\subsubsection{Sanov's theorem}
\label{Sanov's theorem}
Let $L$ be a natural number, and let $\calP$ be the set consisting of the probability distributions on the finite set $ \{1, 2, \dots, L \} $.
\begin{align}
  \calP &= \left\{ (p_1, \dots, p_{L}) \in \bbR_{\geq 0 }^L \ : \ \sumlL p_{\ell} = 1, p_{\ell} \geq 0 \right\}.
\end{align}
$Let$ $q = (q_1, \dots, p_{L}) \in \calP$ be called true distribution, and it is fixed. Let $ X_1, \dots, X_n $ be the random variables independently identically generated from the probability distribution $q$. Also, for each $ \ell = 1, \dots, L $, the random variable $ n_ {\ell} $ is the number of $ X_1, \dots, X_n $ whose value is $ \ell $. Let the empirical distribution $r_n$ be 
\begin{align}
  \label{eq:def:empirical ditribution}
  r_n = \left( \frac{n_1}{n}, \dots, \frac{n_L}{n} \right).
\end{align}
Then, the next theorem holds.

\begin{theorem}[Sanov's Theorem \cite{csiszar2006simple}]
   Let $A$ be a subset of $\calP$. Let $\rmPr(r_n \in A)$ be the probability such that the empirical distribution $r_n$ is included in the set $A $, the following inequality holds:
   \begin{align}
    \limsup_{n \to \infty} \frac{1}{n} \log  \rmPr(r_n \in A) \leq - \inf_{p \in A} D(p\|q).
  \end{align}
\end{theorem}

\subsubsection{the property of the multinomial distribution with one trial}
\label{subsubsec:the property of the multinomial distribution with one trial}

Let $q = (q_1, \dots, q_L) \in \calP$ be positive in all the elements i.e. $ q_1, \dots, q_L > 0 $. Let the positive number $ \varepsilon > 0 $ be sufficiently small and define the set $ \Pe $ as follows:
\begin{align}
  \Pe = \{ p = (p_1, \dots, p_L) : \text{there exists $\ell$ such that $p_\ell \leq \varepsilon$} \}.
\end{align}
Since it is assumed that $\varepsilon$ is sufficiently small, we can take $\varepsilon$ so that $q$ is not included in the set $\Pe$. Due to the $ q \notin \Pe$ and the property of the KL information, $ D (q \| s)> 0 $ holds for all $ s \in \Pe $. Next, the constant $c > 0$ is fixed as an arbitrary number that satisfies $ \displaystyle 0 < c < \inf_{s \in \Pe} D (q \| s) $, and the set $A$ is defined by
\begin{align}
  \label{eq:def:A}
  A = \left\{ r \in \calP : \inf_{s \in \Pe} D(r\|s) > D(r\|q) + \frac{c}{2} \right\}.
\end{align}
Now, we show that $q \in A$ and that $A$ and $\Pe$ have no intersection. Assuming there exists a probability distribution $ r \in A \cap \Pe $, since $r \in A$,
\begin{align}
   \inf_ {s \in \Pe} D (r \| s)> D (r \| q) + \frac{c}{2}.
\end{align}
However, since $ r \in \Pe $, $ \displaystyle \inf_ {s \in \Pe} D (r \| s) = 0 $, which is a contradiction. Furthermore, we can take a sufficient small positive number $ \delta> 0 $ such that
\begin{align}
  \{ r \in \calP : D(r\|q) < \delta \} \subset A.
\end{align}
By the definition of the set A, if $p \in A^c$ then $D(r\|q) \geq \delta$. Applying the Sanov's theorem for the set $A^c$, 
\begin{align}
  \limsup_{n \to \infty} \frac{1}{n} \log  \rmPr(r_n \in A^c) \leq - \inf_{p \in A^c} D(p\|q).
\end{align}
Therefore, for sufficiently large n, 
\begin{align}
  \label{eq:upper bound P(Theta=1)}
  \rmPr(r_n \in A^c) &\leq \exp\left( - n \inf_{p \in A^c} D(p\|q) \right) \\
  &\leq \exp(-n\delta).
\end{align}
Thus, 
\begin{align}
  \rmPr(r_n \in A) \geq 1 - \exp(-n\delta).
\end{align}
By the definition of the set A, with probability at least $1 - \delta$,
\begin{align}
  \label{eq:sample}
  D(r_n \| p) \geq D(r_n \| q) + \frac{c}{2}.
\end{align}
Since $X_1, \dots, X_n$ are the random variables generated independently from the true distribution $q$, and  $n_{\ell}$ is the number of $ X_1, \dots, X_n $ whose value is $ \ell $ , by the Eq.\eqref{eq:sample},
\begin{align}
  \label{eq:sample2}
  \sumlL \frac{n_{\ell}}{n} \log \frac{q_{\ell}}{p_{\ell}} \geq \frac{c}{2}.
\end{align}
By calculating the Eq.\eqref{eq:sample2}, 
\begin{align}
  \label{eq:sample3}
  \prodlL \frac{p_{\ell}^{n_{\ell}}}{q_{j}^{n_{j}}} \leq \exp \left( - \frac{nc}{2} \right).
\end{align}
Eq.\eqref{eq:sample3} means that if $p(x)$ and $q(x)$ are the probability mass functions of the multinomial distribution with one trial respectively, the following inequality holds with the probability at least $1 - \exp (-n \delta)$, 
\begin{align}
  \prodin \frac{p(X_i)}{q(X_i)} \leq \exp \left( - \frac{nc}{2} \right).
\end{align}

\subsubsection{the properties of the predictive distribution of the multinomial mixtures}
\label{the properties of the predictive distribution of the multinomial mixtures}
In this section, we consider the lower bound of the predictive distribution $p(x|X^n)$ when the probability model is the multinomial mixture:
\begin{align}
  p(x|a,b) &= \frac{M!}{\prodlL x_{\ell}!} \sumhH a_h \prodlL b_{h\ell}^{x_{\ell}}.
\end{align}
We introduce the latent variable $ y = (y_1, \dots, y_H) $. The latent variable $y$ is a vector such that one element is 1 and the others are 0. By using the variable $y$, we can rewrite as follows:
\begin{align}
  p(x, y|a, b) &= \frac{M!}{\prodlL x_{\ell}!} \prodhH \left( a_h \prodlL b_{h\ell}^{x_\ell} \right)^{y_h}.
\end{align}
Both the prior distribution of the parameters of the multinomial and that of the parameters of the mixture ratio are Dirichlet distribution:
\begin{align}
  \varphi(a, b|\alpha, \beta) = \frac{1}{R(\alpha, \beta)} \prodhH \left\{ a_{h}^{\alpha_h - 1} \prodlL b_{h\ell}^{\beta_{h\ell}-1} \right\},
\end{align}
where $\alpha, \beta$ are hyperparameter and $R(\alpha, \beta)$ are normalizing constant such that
\begin{align}
  R(\alpha, \beta) = \frac{\Gamma\left( \sumhH \alpha_h \right)}{\prodhH \Gamma(\alpha_h)} \prodhH \frac{\Gamma\left( \sumlL \beta_{h\ell} \right)}{\prodhH \Gamma_{h\ell}}.
\end{align}
To calculate the predictive distribution, we calculate the posterior distribution $p(a, b|X^n)$.
\begin{align}
  &p(a, b|X^n) = \frac{1}{\hat{R}_n}\varphi(a, b|\alpha, \beta) \prodin p(X_i, Y_i|a, b)\\
  &= \frac{1}{\hat{R}_n} \left[ \frac{1}{R(\alpha, \beta)} \prodhH \left\{ a_{h}^{\alpha_h - 1} \prodlL b_{h\ell}^{\beta_{h\ell}-1} \right\} \right] \prodin \left[ \frac{M!}{\prodlL X_{i\ell}!} \prodhH \left( a_h \prodlL b_{h\ell}^{X_{i\ell}} \right)^{Y_{ih}} \right]\\
  &\propto \frac{1}{\hat{R}_n} \prodhH \left\{ a_h^{\alpha_h + \sumin Y_{ih} - 1} \prodlL b_{h\ell}^{\beta_{h\ell} + \sumin X_{i\ell}Y_{ih}-1}\right\},
\end{align}
where $\hat{R}_n$ is normalizing constant of $p(a, b|X^n)$:
\begin{align}
  \hat{R}_n &= \iint \varphi(a, b|\alpha, \beta) \prodin p(X_i, Y_i|a, b) \mathrm{d}a \mathrm{d}b\\
  &\propto \prodhH \left\{ a_h^{\alpha_h + \sumin Y_{ih} - 1} \prodlL b_{h\ell}^{\beta_{h\ell} + \sumin X_{i\ell}Y_{ih}-1}\right\} \mathrm{d}a \mathrm{d}b\\
  &= R\left( \alpha + \sumin Y_i, \beta + \sumin X_i Y_i \right).
\end{align}
Thus, the predictive distribution can be calculated as follows:
\begin{align}
  &p(x|X^n) = \iint p(x, y|w) p(a, b | X^n, Y^n) \mathrm{d}a \mathrm{d}b \\
  &\propto \iint \prodhH \left( a_h \prodlL b_{h\ell}^{x_\ell}\right)^{y_h}   \frac{1}{\hat{R}_n}  \prodhH \left\{ a_h^{\alpha_h + \sumin Y_{ih} - 1} \prodlL b_{h\ell}^{\beta_{h\ell} + \sumin X_{i\ell}Y_{ih}-1}\right\}\mathrm{d}a \mathrm{d}b\\
  &= \frac{1}{\hat{R}_n}  \iint \prodhH\left\{ a_h^{\alpha_h + \sumin Y_{ih} + y_h - 1} \prodlL b_{h\ell}^{\beta_{h\ell} + \sumin X_{i\ell}Y_{ih} + x_{\ell}y_{h} -1} \right\} \mathrm{d}a \mathrm{d}b\\
  &= \frac{R(\alpha + \sumin Y_i + y, \beta + \sumin X_{i}Y_{i} + xy)}{R\left( \alpha + \sumin Y_i, \beta + \sumin X_i Y_i \right)}.
\end{align}
Since the latent variable $y$ is a vector such that one element is 1 and since the others are 0 and the vector $x$ satisfies $\displaystyle \sumlL x_{\ell} = 1$, apply the property of the Gamma function $\Gamma(x+1) = x\Gamma(x)$, we can show that
\begin{align}
  p(x, y|X^n, Y^n) = O\left( \frac{1}{n^M} \right).
\end{align}
Thus, there exists a constant positive number $C > 0$ such that for all $x, y$,
\begin{align}
  p(x, y|X^n, Y^n) > \frac{C}{n^M}.
\end{align}
By the definition of the marginal distribution,
\begin{align}
  \label{eq:order predictive}
  p(x|X^n) &= \sum_{y}\sum_{Y^n} p(x, y | X^n, Y^n) > \frac{C}{n^M}.
\end{align}
Therefore,
\begin{align}
  \label{eq:order log predictive}
  \log p(x|X^n) &\leq \log \frac{C}{n^M} = - M \log (Cn).
\end{align}
In this section, the prior distribution of parameters is discussed as the Dirichlet distribution. In the main theorem, we also consider the case where a non-zero and bounded prior distribution. In the case of such prior distribution, since any distribution does not affect the poles of the zeta function, the lower bound of the predicted distribution $p(x| X^n)$ is not in the exponential order as in the above discussion.

\subsubsection{properties of the maximum likelihood estimator of multinomial mixtures}
\label{subsub:properties of the maximum likelihood estimator of multinomial mixtures}

Multinomial mixtures with $M$ trials are finite distributions on $\bbZ^{L}_{\geq 0}$ such that $x_1 + \dots + x_L = M$. The set $\calJ$ is defined by 
\begin{align}
  \label{eq:def:calJ}
  \calJ = \left\{ (x_1, \dots, x_L) \in \bbZ : \sumlL x_\ell = M, \ x_{\ell} \geq 0 \right\}.
\end{align}
The set $\calJ$ is finite, and the number of elements of $\calJ$ is $J$. Let $\calP_j$ be the set of all discrete probability distributions on the finite set $\{1, 2, \dots, J \}$:
\begin{align}
  \label{eq:def:calPJ}
  \calP_J &= \left\{ (p_1, \dots, p_{J}) \in \bbR_{\geq 0 }^J \ : \ \sumlL p_{j} = 1, p_{j} \geq 0 \right\}.
\end{align}
The probability distribution on $\calJ$ that can be expressed by multinomial mixtures of $M$ trials is included in the set $\calP_J$. Given that the probability model $p(x|w)$ and the true distribution $q(x)$ are both multinomial mixtures with $M$ trials, and that the corresponding distributions in the probability distribution on $\calJ$ are the $\pbar(x|w)$ and $\qbar(x)$, the average error function $K(w)$ is expressed as follows:
\begin{align}
  \label{eq:K(w) on calJ}
  K(w) &= \sumx \qbar(x) \log \frac{\qbar(x)}{\pbar(x|w)}.
\end{align}
Since the function $K(w)$ is the KL information between the true distribution $q(x)$ and the probability model $p(x|w)$, if $p(x|w) = q(x)$ then $K(w) = 0$, and otherwise $K(w) > 0$. Nowhere, we fix the positive constant $c > 0$ in the section \ref{subsubsec:the property of the multinomial distribution with one trial}, and define the subset $A$ of $\calP_J$ as follows:
\begin{align}
  \label{eq:def:A on calPJ}
  A = \left\{ p(x|w) : K(w) < \frac{c}{2} \right\}.
\end{align}
Since $q(x)$ are positive for all the points $x \in D$ from the assumption, by fixing the positive constant $\varepsilon > 0$ in the section \ref{subsubsec:the property of the multinomial distribution with one trial}, the subset $E$ of $\calP_J$ is defined as follows:
\begin{align}
  \label{eq:def:E on calPJ}
  E = \{ p(x|w) : \exists x \in \{ 1, \dots, J \} \mbox{ s.t. } \pbar(x|w) < \varepsilon \}.
\end{align}
The $E$ can be defined so that the sets $E$ and $A$ have no intersection. Then, the log empirical loss $L_n$ can be calculated as follows:
\begin{align}
  L_n &= \frac{1}{n} \sumin \log \frac{\qbar(X_i)}{\pbar(X_i)}\\
  &= \sumjJ \frac{n_j}{n} \log \frac{\qbar_j}{\pbar_j}\\
  \label{eq:calcLn}
  &= - \sumjJ \frac{n_j}{n} \log \frac{\pbar_j}{n_j/n} - \sumjJ \frac{n_j}{n} \log \frac{n_j/n}{\qbar_j}.
\end{align}
Since if $n \to \infty$ then $\frac{n_j}{n} \to \qbar_j$, The first term of the Eq.\eqref{eq:calcLn} converges to a certain constant, and the second term converges to 0. Thus, with probability at least $1 - \exp(-n\delta)$,
\begin{align}
  \frac{1}{n} \log \prodin \frac{q(X_i)}{p(X_i)} &\leq -\frac{c}{2}.
\end{align}
Therefore,
\begin{align}
  \prodin \frac{\qbar(X_i)}{\pbar(X_i)} &\leq \exp(-\frac{nc}{2}).
\end{align}
Let $w_0$ be a maximum likelihood estimator from multinomial mixtures with $M$ trials, and let $w_1$ be also a maximum likelihood estimator from all the discrete distributions on $\calP_J$. Since the set of distributions that can be expressed by multinomial mixtures with $M$ trials include in the set of distributions that can be expressed by all the discrete distributions on $\calP_J$, $\displaystyle \prodin p(X_i | w_0) \leq \prodin p(X_i | w_1)$ holds. Therefore, with probability at least $1 - \exp(-n\delta)$,
\begin{align}
  \label{eq:the upper bound MLE}
  \prodin \frac{p(X_i|w_0)}{q(X_i)} \leq \frac{\pbar(X_i|w_1)}{\qbar(X_i)} \leq \exp\left( - \frac{nc}{2} \right).
\end{align}

\subsubsection{the proof of lemma \ref{Lemma of Theta}}
\label{subsub:the proof of lemma of theta}
\begin{proof}
We fix the positive constants $\varepsilon, c, \delta > 0$ in the section \ref{subsubsec:the property of the multinomial distribution with one trial}, and we define the set $A$ in the Eq.\eqref{eq:def:A on calPJ}. The random variable $\Theta(X^n)$ is defined as follows:
\begin{align}
  \label{eq:def:Theta}
  \Theta(X^n) = 
  \begin{cases}
    1 & (r_n \in A)\\
    0 & (r_n \notin A)
  \end{cases},
\end{align}
where $r_n$ is the empirical distribution defined the Eq.\eqref{eq:def:empirical ditribution}. From the Eq.\eqref{eq:upper bound P(Theta=1)}, the probability of $\Theta(X^n) = 1$ is less than $(\exp(-n\delta))$. Using that true distribution $q(X)$ does not depend on the sample size $n$ and using the Eq.\eqref{eq:order log predictive},
\begin{align}
   \bbE\left[ \Theta(X^{n+1}) \log \frac{Z_1(X^{n}) + Z_2(X^{n})}{Z_1(X^{n+1})}\right]
   &= \bbE\left[ \Theta(X^{n+1}) \frac{q(X)}{p(x|X^n)}\right]\\
   &= O\left( \exp\left( -n\delta \right) \left(- M \log (Cn) \right) \right)\\
   \label{eq:lemma_eq1}
   &= O(\exp(-n)).
\end{align}
Moreover, by the Eq.\eqref{eq:the upper bound MLE},
\begin{align}
  \begin{split}
    &\bbE\left[ \left( 1 - \Theta(X^{n+1})\right) \log \frac{Z_1(X^{n}) + Z_2(X^{n})}{Z_1(X^{n+1}) + Z_2(X^{n+1})}  \right] \\
    &\quad \leq \bbE\left[  \log \frac{\frac{(\log n)^{m-1}}{n^{\lambda}}Z_0(X^n) + \exp(-nc/2) }{\frac{(\log (n+1))^{m-1}}{(n+1)^{\lambda}}Z_0\left(X^{n+1}\right) + \exp\left(-(n+1)c/2\right)}  \right]
  \end{split}\\
  \label{eq:lemma_eq2}
  &= \frac{\lambda}{n} - \frac{m-1}{n\log n} + \bbE\left[ \log \frac{Z_0(X^n)}{Z_0(X^{n+1})} \right] + o\left( \frac{1}{n\log n} \right).
\end{align}

By the reference \cite{watanabe2018mathematical}, 
\begin{align}
  \bbE\left[ \log \frac{Z_0(X^n)}{Z_0(X^{n+1})} \right] = o\left( \frac{1}{n\log n} \right).
\end{align}
Therefore, 
\begin{align}
  \begin{split}
    &\bbE\left[ \left( 1 - \Theta(X^{n+1})\right) \log \frac{Z_1(X^{n}) + Z_2(X^{n})}{Z_1(X^{n+1}) + Z_2(X^{n+1})}  \right] \\
  &\quad \leq \frac{\lambda}{n} - \frac{m-1}{n\log n} + o\left( \frac{1}{n\log n} \right).
  \end{split}
\end{align}
From the above, the lemma \ref{Lemma of Theta} is shown.
\end{proof}

%% file: properties_general_components.tex
We prepare a lemma that holds for multinomial mixtures of general components.

%%%%%%%%%%%%%%%% Lemma 6.4
\begin{lemma}
\label{lem:general multinomial mixtures}
Let $p(x|w)$ be a multinomial mixture with $H$ components:
\begin{align}
  \label{lem:eq:p(x|w)}
  p(x|w) = \sumhH a_h \rmMul_L(x|b_{h}) \ \ (w \in W).
\end{align}
Also, let the true distribution $q(x)$ be a multinomial mixture with $H^* \ (H \geq H^*)$ components:
\begin{align}
  \label{lem:eq:q(x)}
  q(x) = \sumhHs \ahs \rmMul_L(x| \bhs).
\end{align}
Let $K(w)$ be the mean error function determined by the probability model $p(x|w)$ and the true distribution $q(x)$. $K(w)$ has the same RLCT and the multiplicity as $K_1 (w)$ defined below:
\begin{align}
  \label{lem:eq:K1(w)}
	K_1(w) &= \sumxD \left\{ \sumhH \ah \left( \prodlLminus \bhl^{\xl} \right) - \sumhsHs \ahs \left( \prodlLminus (\bhls)^{\xl} \right) \right\}^2,
\end{align}
that is, $K(w) \sim K_1(w)$.
\end{lemma}

\begin{proof} From the lemma \ref{lem:zeta}, the mean error function $K(w)$ is equivalent to $K'(w)$, that is, the RLCT and the multiplicity are equal. Calculate $p(x|w) - q(x)$ as follows:
\begin{align}
  \label{lem:proof:p(x|w)-q(x):1}
  p(x|w) - q(x) &= \left\{ \sumhH \ah \frac{M!}{\prodlL \xl!} \prodlL \bhl^{\xl} - \sumhHs \ahs \frac{M!}{\prodlL \xl!}  (\bhls)^{\xl} \right\} \\
  \label{lem:proof:p(x|w)-q(x):2}
  &= \left( \frac{M!}{\prodlL \xl!} \right) \left\{ \sumhH \ah \prodlL \bhl^{\xl} - \sumhHs \ahs \prodlL (\bhls)^{\xl}  \right\}. 
\end{align}
Since $\displaystyle \frac{M!}{\prodlL \xl!}$ is greater than $0$ and bounded for any $x \in D$,
\begin{align}
  \label{lem:proof:eliminate constants}
  K(w) \sim \sumxD \left\{ \sumhH \ah \prodlL \bhl^{\xl} - \sumhHs \ahs \prodlL (\bhls)^{\xl}  \right\}^2.
\end{align}
Furthermore, since for each $h \in [H]$ and for each $h \in [\Hs]$ :  $\displaystyle \sumlL \bhl = 1$, $\sumlL \bhls = 1$, both $\bhL$ and $\bhL^*$ can be represented by
other $\bhl$ and $\bhls$ for each $h \in [H]$, 
\begin{align}
  \label{lem:proof:eliminate bL}
  \begin{split}
	  &\sumhH \ah \prodlL \bhl^{\xl} - \sumhHs \ahs \prodlL (\bhls)^{\xl} \\
      &\quad =  \sumhH \ah \left( \prodlLminus \bhl^{\xl} \right) \left( 1 - \sumlLminus \bhl \right)^{\xL} - \sumhHs \ahs \left( \prodlLminus (\bhls)^{\xl} \right) \left( 1 - \sumlLminus \bhls \right)^{\xL}.
  \end{split}
\end{align}
Here, by using the binomial theorem,
\begin{align}
  \label{lem:proof:using the binomial theorem:1}
  \left( 1 - \sumlLminus \bhl \right)^{\xL} &= \sum_{i=0}^{\xL} \binom{\xL}{i} \left( - \sumlLminus \bhl \right)^i \ 1^{\xL - i}\\
  \label{lem:proof:using the binomial theorem:2}
  &= 1 +  \sum_{i=1}^{\xL} \binom{\xL}{i} (-1)^{i} \left( \sumlLminus \bhl \right)^{i}.
\end{align}
Also
\begin{align}
  \label{lem:proof:using the binomial theorem:3}
  \left( 1 - \sumlLminus \bhls \right)^{\xL} &= 1 +  \sum_{i=1}^{\xL} \binom{\xL}{i} (-1)^{i} \left( \sumlLminus \bhls \right)^{i}.
\end{align}
Therefore, 
\begin{align}
  \label{lem:proof:using the binomial theorem:4}
  \begin{split}
	&\sumhH \ah \prodlL \bhl^{\xl} - \sumhHs \ahs \prodlL (\bhls)^{\xl} \\
    &\quad = \sumhH \ah \left( \prodlLminus \bhl^{\xl} \right) \left\{ 1 +  \sum_{i=1}^{\xL} \binom{\xL}{i} (-1)^{i} \left( \sumlLminus \bhl \right)^{i} \right\}\\
    &\quad -  \sumhHs \ahs \left( \prodlLminus (\bhls)^{\xl} \right) \left\{ 1 +  \sum_{i=1}^{\xL} \binom{\xL}{i} (-1)^{i} \left( \sumlLminus \bhls \right)^{i} \right\}
  \end{split}\\
  \label{lem:proof:using the binomial theorem:5}
  \begin{split}
    &= \sumhHs \ah \prodlLminus (\bhl)^{\xl} - \sumhHs \ahs \prodlLminus (\bhls)^{\xl} \\
    &\quad + \sumhH \ahs \left( \prodlLminus (\bhl)^{\xl} \right) \sum_{i=1}^{\xL} \binom{\xL}{i} (-1)^{i} \left( \sumlLminus \bhl \right)^{i}\\
    &\quad -  \sumhHs \ahs \left( \prodlLminus (\bhls)^{\xl} \right) \sum_{i=1}^{\xL} \binom{\xL}{i} (-1)^{i} \left( \sumlLminus \bhls \right)^{i}.
  \end{split}
\end{align}
For simplicity, for each $h \in [H]$ and $h \in [\Hs]$,  
we define $\Ah, \Ahs$ as follows:
\begin{align}
  \label{lem:proof:def:Ah}
	\Ah &= \ah \prodlLminus (\bhl)^{\xl} \ \ (h = 1, \dots, H),\\ 
  \label{lem:proof:def:Ahs}
    \Ahs &= \ahs \left( \prodlLminus (\bhls)^{\xl} \right) \ \ (h = 1, \dots, \Hs).
\end{align}
It follows that
\begin{align}
  \label{lem:proof:Ah_calc}
  \begin{split}
	&\sumhH \ah \prodlL \bhl^{\xl} - \sumhHs \ahs \prodlL (\bhls)^{\xl} \\
    &= \sumhH \Ah - \sumhHs \Ahs \\
    &\quad + \sumhH \Ah \sum_{i=1}^{\xL} \binom{\xL}{i} (-1)^{i} \left( \sumlLminus \bhl \right)^{i}  - \sumhHs \Ahs \sum_{i=1}^{\xL} \binom{\xL}{i} (-1)^{i} \left( \sumlLminus \bhls \right)^{i}.
  \end{split}
\end{align}
By using the multinomial theorem, 
\begin{align}
  \label{lem:proof:using the multinomial theorem:1}
  \left( \sumlLminus \bhl \right)^{i} &= \summul \frac{i!}{i_1! \dots i_{L-1}!} b_{h1}^{i_1} \dots b_{hL-1}^{i_{L-1}},\\
  \label{lem:proof:using the multinomial theorem:2}
  \left( \sumlLminus \bhls \right)^{i} &= \summul \frac{i!}{i_1! \dots i_{L-1}!} (\bs_{h1})^{i_1} \dots (\bs_{hL-1})^{i_{L-1}},
\end{align}
where
the summation $\displaystyle \summul$ shows the sum of all over
 non-negative integer sets $(i_1, \dots, i_{L-1})$ such that $i_{1} + \dots + i_{L-1} = i$. 
We apply Eqs.\eqref{lem:proof:using the multinomial theorem:1} and 
\eqref{lem:proof:using the multinomial theorem:2} to the Eq.\eqref{lem:proof:Ah_calc}. 
Let $C_{i_{\ell}}$ be
 $\displaystyle \frac{i!}{i_1! \dots i_{L-1}!}$. Then we obtain 
\begin{align}
  \label{lem:proof:calc1}
  \begin{split}
    &\sumhH \ah \prodlL \bhl^{\xl} - \sumhHs \ahs \prodlL (\bhls)^{\xl} \\
    &= \sumhH \Ah - \sumhHs \Ahs \\
    &\quad + \sumhH \Ah \sum_{i=1}^{\xL} \binom{\xL}{i} (-1)^{i}  \summul C_{i_{\ell}} b_{h1}^{i_1} \dots b_{hL-1}^{i_{L-1}}\\
    &\quad - \sumhHs \Ahs \sum_{i=1}^{\xL} \binom{\xL}{i} (-1)^{i} \summul C_{i_{\ell}} (\bs_{h1})^{i_1} \dots (\bs_{hL-1})^{i_{L-1}}
  \end{split}\\
  \label{lem:proof:calc2}
  \begin{split}
    &= \sumhH \Ah - \sumhHs \Ahs \\
    &\quad + \sum_{i=1}^{\xL}\binom{\xL}{i} \summul (-1)^{i}  C_{i_{\ell}} \binom{\xL}{i} \\
    &\quad \times \left\{ \sumhH \Ah b_{h1}^{i_1} \dots b_{hL-1}^{i_{L-1}} - \sumhHs \Ahs (\bs_{h1})^{i_1} \dots (\bs_{hL-1})^{i_{L-1}} \right\}.
  \end{split}
\end{align}
By using Eqs.(\ref{lem:proof:def:Ah}) and (\ref{lem:proof:def:Ahs}), 
\begin{align}
  \label{lem:proof:calc3}
  \begin{split}
    &\sumhH \ah \prodlL \bhl^{\xl} - \sumhHs \ahs \prodlL (\bhls)^{\xl} \\
    &= \sumhH \Ah - \sumhHs \Ahs \\
    & + \sum_{i=1}^{\xL} \binom{\xL}{i} \summul (-1)^{i}  C_{i_{\ell}} \binom{\xL}{i} \\
& \times
\left\{ \sumhH \ah b_{h1}^{x_1 + i_1} \dots b_{hL-1}^{x_{L-1} + i_{L-1}} 
 - \sumhHs \ahs (\bs_{h1})^{x_1 + i_1} \dots (\bs_{hL-1})^{x_{L-1} + i_{L-1}} \right\}.
  \end{split}
\end{align}
We introduce a polynomial $f_{L-1}(x_1, \dots, x_{L-1} ; w)$ defined by 
\begin{align}
  \label{lem:proof:def:f}
  f_{L-1}(x_1, \dots, x_{L-1} ; w) &= \sumhH \Ah - \sumhHs \Ahs,\\
  \label{eq:proof:def:f}
&= \sumhH \ah \prodlLminus (\bhl)^{\xl} - \sumhHs \ahs \left( \prodlLminus (\bhls)^{\xl} \right).
\end{align}
Then by Eq.(\ref{lem:proof:calc3}), it follows that 
\begin{align}
  \label{lem:proof:calc4}
  \begin{split}
    &\sumhH \ah \prodlL \bhl^{\xl} - \sumhHs \ahs \prodlL (\bhls)^{\xl} \\
    &\quad = f_{L-1}(x_1, \dots, x_{L-1} ; w) \\
&+ \sum_{i=1}^{\xL} \binom{\xL}{i} \summul (-1)^{i}  C_{i_{\ell}} \binom{\xL}{i} f_{L-1}(x_1+i_1 , \dots, x_{L-1}+i_{L-1} ; w).
  \end{split}
\end{align}
The second term of Eq.\eqref{lem:proof:calc4} can be expressed by the linear sum of the first term $ f_{L-1}(x_1, \dots, x_ {L-1};w) $. That is, in the second term, there is a constant $ C'(x_1, \dots, x_ {L-1}) $ that does not depend on parameters, 
\begin{align}
  \label{lem:proof:calc5}
  \begin{split}
    &\sum_{i=1}^{\xL} \binom{\xL}{i} \summul (-1)^{i}  C_{i_{\ell}} \binom{\xL}{i}
 f_{L-1}(x_1+i_1 , \dots, x_{L-1}+i_{L-1} ; w)\\
    &\quad = \sum_{x \in D} C'(x_1, \dots, x_{L-1}) f_{L-1}(x_1, \dots, x_{L-1} ; w).
  \end{split}
\end{align}
Therefore, the ideal generated from the set $\{ f_{L-1}(x_1, \dots, x_{L-1}) \}_{x \in D}$ and the ideal generated from the set $\left\{ f_{L-1}(x_1, \dots, x_{L-1}) + f_{L-1}(x_1+i_1, \dots, x_{L-1}+i_{L-1}) \right\}_{x \in D}$ are equal, so the function $K_1(w)$ is defined as follows:
\begin{align}
	K_1(w) &= \sumxD f_{L-1}(x_1, \dots, x_{L-1} ; w)^2 \\
& = \sumxD \left\{ \sumhH \ah \left( \prodlLminus (\bhl)^{\xl} \right) - \sumhHs \ahs \left( \prodlLminus (\bhls)^{\xl} \right) \right\}^2 .
  \label{eq:proof:def:K1}
\end{align}
From the lemma \ref{lem:zeta} \eqref{lem:ideal}, the two functions $K(w)$ and $K_1(w)$ are equivalent, that is, their RLCTs and multiplicities are equal.
\end{proof}

%% file: properties_two_components.tex
So far, we have prepared the lemma \ref{lem:general multinomial mixtures}, which holds for multinomial mixtures with general components. Hereafter, we assume that the number of components of the multinomial mixtures of the probability model is 2 (i.e. $H = 2$) and that the true distribution is the multinomial distribution (i.e. $H^* = 1$). That is, the probability model $p(x|w)$ and the true distribution $q(x)$ are as follows:
\begin{align}
  \label{eq:assume:p(x|w)}
  p(x|w) &= a\rmMul_{L}(x|b) + (1-a)\rmMul_{L}(x|c), \ b, c \in B,  \\
  \label{eq:assume:q(x)}
  q(x) &= \rmMul_{L}(x|b^*), \ b^* \in B, \ \prodlL b^*_{\ell} \neq 0.
\end{align}
Then, the polynomial $f_L(x_1, \cdots, x_{L})$ defined by the Eq.\eqref{lem:proof:def:f} is expressed as follows:
\begin{align*}
  f_L(x_1, \cdots, x_{L}; w) = a \prod_{\ell=1}^{L}b_\ell^{x_\ell}  + (1-a) \prod_{\ell=1}^{L}c_\ell^{x_\ell} - \prod_{\ell=1}^L (b_\ell^*)^{x_\ell}.
\end{align*}

%%%%%%%%%%%%%%%%%% Lemma 6.5
\begin{lemma}
\label{lem:3}
  For $j \in [2:L-1]$, the following holds:
  \begin{align}
    \label{eq:lem_f}
    \begin{split}
       &f_{L-1}(x_1, \cdots, x_j, \cdots,  x_{L-1}; w)\\
       &= (b_j + c_j) f_{L-1}(x_1, \cdots, x_j - 1, \cdots, x_{L-1} ; w) \\
&\quad- b_jc_j f_{L-1}(x_1, \cdots, x_j - 2, \cdots, x_{L-1} ; w)\\
       &\quad-  (b_j - b_j^*)(c_j - b_j^*)(b_j^*)^{x_j-2} \prodij (b_\ell^*)^{x_\ell},
    \end{split}  
  \end{align}
  where $[2:L-1]$ represents the set $\{ \ell \in \bbZ : 2 \leq \ell \leq L-1 \}$.
\end{lemma}
  
\begin{proof}
By calculating each of the three terms in Eq.(\ref{eq:lem_f}), 
\begin{align}
  \label{lem:proof:fst}
  \begin{split}
    & (b_j + c_j) f_{L-1}(x_1, \cdots, x_j - 1, \cdots, x_{L-1} ; w) \\
   &\quad= a \prodiLL b_\ell^{x_\ell} + a c_j b_j^{x_j-1} \prodij b_\ell^{x_\ell}   + (1-a) b_j c_j^{x_j-1} \prodij c_\ell^{x_\ell}\\
    &\quad + (1-a) \prodiLL c_\ell^{x_\ell} - (b_j + c_j) (b_j^*)^{x_j -1} \prodij (b_\ell^*)^{x_\ell},
  \end{split}\\
  \label{lem:proof:snd}
  \begin{split}
   &  b_jc_j f_{L-1}(x_1, \cdots, x_j - 2, \cdots, x_{L-1} ; w) \\
&\quad=  - a c_j b_j^{x_j - 1}  \prodij b_\ell^{x_\ell}  - (1-a) b_j c_j^{x_j-1} \prodij c_\ell^{x_\ell},
  \end{split}\\
  \label{lem:proof:trd}
  \begin{split}
  & (b_j - b_j^*)(c_j - b_j^*)(b_j^*)^{x_j-2} \prodij (b_\ell^*)^{x_\ell} \\
&\quad= -b_jc_j(b_j^*)^{x_j-2}\prodij(b_\ell^*)^{x_\ell} + b_j(b_j^*)^{x_j-1}\prodij(b_\ell^*)^{x_\ell}\\
    &\quad + c_j(b_j^*)^{x_j-1}\prodij(b_\ell)^{x_\ell} - \prodij(b_\ell^*)^{x_\ell}.
  \end{split}
\end{align}
Then
Eq.\eqref{eq:lem_f} is obtained by summing the above three Eqs.\eqref{lem:proof:fst}, \eqref{lem:proof:snd}, \eqref{lem:proof:trd}.
\end{proof}

%%%%%%%%%%%%%%%%%%% Lemma 6.6
\begin{lemma}
\label{lem:4}
Define the set $D'$ of $x$ as follows:
\begin{align}
  \label{eq:def:D'}
  D' = \bigl\{ (x_1, x_2, \cdots, x_{L-1}) \in \mathbb Z^{L-1} \ | \  x_\ell \in \{0, 1\} \bigr\}
\end{align}
Define the function $K_2(w)$ of parameter $w$ as follows:
\begin{align}
  \label{eq:def:K2}
  K_2(w) = \sum_{\ell=1}^{L-1}(b_\ell - b_\ell^*)^2(c_\ell - b_\ell^*)^2 + \sum_{x \in D'} f_{L-1}(x_1, \cdots, x_{L-1}; w)^2.
\end{align}
Then, $K(w) \sim K_2(w)$.
\end{lemma}
\begin{proof}
By using lemma \ref{lem:3} inductively, $f_{L-1}(x_1, \cdots,  x_{L-1}; w)$ can be expressed using
\begin{itemize}
	\item $f_{L-1}(1, 0, 0, \cdots, 0; w),  f(0, 1, 0, \cdots, 0; w), \cdots$
	\item $f_{L-1}(1, 1, 0, \cdots, 0; w),  f(1, 0, 1, \cdots, 0; w), \cdots$
	\item $\cdots$
	\item $f_{L-1}(1, 1, \cdots, 1; w)$
    \item $(b_\ell - b_\ell^*)(c_\ell - b_\ell^*).$
\end{itemize}
Since the ideal generated from $\{ f_{L-1}(x_1, \cdots, x_{L-1}; w) \}_{x \in D} \cup \{  (b_i-b_i^*)^2(c_i - b_i^*)^2 \}_{\ell=1}^{L}$ and the ideal generated from $\{ f_{L-1}(x_1, \cdots, x_{L-1}; w) \}_{x \in D'} \cup \{  (b_i-b_i^*)^2(c_i - b_i^*)^2 \}_{\ell=1}^{L}$ are equal, lemma \ref{lem:4} can be applied. Thus, $K(w) \sim K_2(w)$.
\end{proof}

%% file: main_proof.tex
Let us prove Theorem.\ref{maintheorem}.

\begin{proof} (Proof of Theorem.\ref{maintheorem})
 From lemma \ref{lem:4}, to clarify the RLCT and multiplicity of the multinomial mixtures with two components, we calculate the largest pole of the zeta function determined by $K_2(w)$ and $\varphi(w)$. 
\begin{align}
  \label{eq:calc_K2}
  \begin{split}
    K_2(w) &= \sumiLL (b_\ell - b_\ell^*)^2(c_\ell - b_\ell^*)^2 \\
    &\quad + \sum_{\ell_1 \in \{1, \cdots L-1 \}} (ab_{\ell_1} + (1-a)c_{\ell_1} - b^*_{\ell_1})^2\\
    &\quad + \sum_{\ell_1, \ell_2 \in \{1, \cdots, L-1 \}} (ab_{\ell_1}b_{\ell_2} + (1-a)c_{\ell_1}c_{\ell_2} - b^*_{\ell_1}b^*_{\ell_2})^2\\
    &\quad + \dots\\
    &\quad + \sum_{\ell_1, \cdots, \ell_{L-1} \in  \{1, \cdots, L-1 \}} \Bigl(a \prod_{k=1}^{L-1}b_{\ell_k} + (1-a) \prod_{k=1}^{L-1}c_{\ell_k} - \prod_{k=1}^{L-1} b^*_{\ell_k} \Bigr)^2,
  \end{split}
\end{align}
where the summation $\displaystyle \sum_{\ell_1, \ell_2, \dots, \ell_i \in \{1, \cdots, L-1 \}}$ represents the sum of the sets $\{ (\ell_1, \dots, \ell_i) \in \{1, \dots, L-1 \}^i :  j \neq j' \Rightarrow \ell_j \neq \ell_{j'} \ (\forall j, j' \in [i])  \}$. 
Let us define a map $\Phi_1: u \mapsto w$, where  for  each $\ell \in [L-1]$, 
\begin{align}
  \label{eq:def:Phi1}
  \begin{cases}
    B_{\ell} = b_{\ell}^* \\
    \beta_{\ell} = b_\ell - B_{\ell}\\
    \gamma_{\ell} = c_\ell - B_{\ell}
 \end{cases}.
\end{align}
The parameter $u$ consists of $(a, \beta, \gamma)$, where $\beta = (\beta_1, \dots, \beta_{L-1}),\ \gamma = (\gamma_1, \dots, \gamma_{L-1})$. Based on the map,
\begin{align}
  \label{eq:calc_K2(Phi1)}
  \begin{split}
    &K_2(\Phi_1(u))\\
    &= \sum_{\ell=1}^{L-1} \beta_\ell^2 \gamma_\ell^2\\
    &\quad + \sum_{\ell_1 \in \{1, \cdots L-1 \}} (a \beta_{\ell_1} + (1-a) \gamma_{\ell_1})^2\\
    &\quad + \sum_{\ell_1, \ell_2 \in \{1, \cdots, L-1 \}}  \Bigl\{ a\beta_{\ell_1}\beta_{\ell_2} + (1-a)\gamma_{\ell_1}\gamma_{\ell_2} \\
&\quad + B_{\ell_2} \{a\beta_{\ell_1} +(1-a) \gamma_{\ell_1}\}
+ B_{\ell_1}\{a\beta_{\ell_2} + (1-a)\gamma_{\ell_2}\} \Bigr\}^2\\
    &\quad +  \dots \\
    &\quad + \sum_{\ell_1, \cdots, \ell_{L-1} \in  \{1, \cdots, L-1 \}} \Bigl\{a\prod_{k=1}^{L-1} \beta_{\ell_k} + (1-a)\prod_{k=1}^{L-1}\gamma_{\ell_k} \\
 & \quad + \cdots + \sum_{j=1}^{L-1}\bigl(\prod_{k \neq j}^{L-1} B_{\ell_{k}} \bigr)(a \beta_j + (1-a)\gamma_j) \Bigr\}^2.
  \end{split}
\end{align}
From the symmetry of the parameters, we can restrict the integration range for $a$ as $0 \leq a \leq \frac{1}{2}$
without loss of generality. 
Let us define a map $\Phi_2 : v \to u$, where 
\begin{align}
  \label{eq:def:Phi2}
  \delta_{\ell} = a\beta_{\ell} + (1-a)\gamma_{\ell} \ \ (\ell = 1, \cdots, L-1).
\end{align}
The parameter $v$ consists of $(a, \beta, \delta)$, where $\delta = (\delta_1, \dots, \delta_{L-1})$.
Since we consider the range $\frac{1}{2} \leq 1-a \leq 1$, the Jacobian determinant of this transform is not equal to zero. Therefore, neither the maximum pole nor its order of the zeta function changes. 
We can obtain that
\begin{align}
  \label{eq:calc_K2(Phi2Phi1)1}
  \begin{split}        
    &K_2(\Phi_2(\Phi_1(v)))\\
    &= \sum_{\ell=1}^{L-1} \delta_{\ell}^2
+ \sum_{\ell=1}^{L-1} \beta_\ell^2 \gamma_\ell^2\\
    &\quad + \sum_{\ell_1, \ell_2 \in \{1, \cdots, L-1 \}}  \Bigl\{ a\beta_{\ell_1}\beta_{\ell_2} + (1-a)\gamma_{\ell_1}\gamma_{\ell_2} \\
& \quad+ B_{\ell_2} \{a\beta_{\ell_1} +(1-a) \gamma_{\ell_1}\} 
 + B_{\ell_1}\{a\beta_{\ell_2} + (1-a)\gamma_{\ell_2}\} \Bigr\}^2\\
    &\quad +  \dots \\
    &\quad + \sum_{\ell_1, \cdots, \ell_{L-1} \in  \{1, \cdots, L-1 \}} \Bigl\{a\prod_{k=1}^{L-1} \beta_{\ell_k} + (1-a)\prod_{k=1}^{L-1}\gamma_{\ell_k} \\
&\quad + \cdots + \sum_{j=1}^{L-1}\bigl(\prod_{k \neq j}^{L-1} B_{\ell_{k}} \bigr)(a \beta_j + (1-a)\gamma_j) \Bigr\}^2.
  \end{split}
\end{align}
Here, we eliminate $\gamma$ by using $\displaystyle \gamma_{\ell} = \frac{\delta_{\ell} - a\beta_{\ell}}{1-a}$,
\begin{align}
  \label{eq:eliminate gamma1}
  \sum_{\ell=1}^{L-1} \beta_\ell^2 \gamma_\ell^2 &= \sum_{\ell=1}^{L-1} \beta_{\ell}^2  \Bigl( \frac{\delta_{\ell} - a\beta_{\ell}}{1-a} \Bigr)^2\\
  \label{eq:eliminate gamma2}
  &= \frac{1}{(1-a)^2} \sum_{\ell=1}^{L-1} \beta_{\ell}^2(\delta_{\ell} - a\beta_{\ell} )^2.
\end{align}
Since $\frac{1}{2} \leq 1-a \leq 1$ and lemma \ref{lem:zeta}\eqref{lem:bounded},
\begin{align}
  \label{eq:calc_K2(Phi2Phi1)2}
  \sum_{\ell=1}^{L-1} \delta_{\ell}^2 + \sum_{\ell=1}^{L-1} \beta_\ell^2 \gamma_\ell^2  &= \sum_{\ell=1}^{L-1} \delta_{\ell}^2 + \frac{1}{(1-a)^2} \sum_{\ell=1}^{L-1} \beta_{\ell}^2(\delta_{\ell} - a\beta_{\ell} )^2\\ 
  \label{eq:calc_K2(Phi2Phi1)2.5}
  &\sim \sum_{\ell=1}^{L-1} \delta_{\ell}^2  + \sum_{\ell=1}^{L-1}  \beta_{\ell}^2 (-a\beta_{\ell})^2\\
  \label{eq:calc_K2(Phi2Phi1)3}
  &=  \sum_{\ell=1}^{L-1} \delta_{\ell}^2 +  \sum_{\ell=1}^{L-1}  a^2\beta_{\ell}^4.
\end{align}
Moreover, 
\begin{align}
  \label{eq:calc_K2(Phi2Phi1)4}
  \begin{split}
    &\sum_{\ell=1}^{L-1} \delta_{\ell}^2 +  \sum_{\ell=1}^{L-1}  a^2\beta_{\ell}^4\\
    &\quad + \sum_{\ell_1, \ell_2 \in \{1, \cdots, L-1 \}}  \Bigl\{ a\beta_{\ell_1}\beta_{\ell_2} + (1-a)\gamma_{\ell_1}\gamma_{\ell_2} \\
&\quad+ B_{\ell_2} \{a\beta_{\ell_1}
+(1-a) \gamma_{\ell_1}\} + B_{\ell_1}\{a\beta_{\ell_2} + (1-a)\gamma_{\ell_2}\} \Bigr\}^2
  \end{split}\\
  \label{eq:calc_K2(Phi2Phi1)5}
    &= \sum_{\ell=1}^{L-1} \delta_{\ell}^2 +  \sum_{\ell=1}^{L-1}  a^2\beta_{\ell}^4 +  \sum_{\ell_1, \ell_2 \in \{1, \cdots, L-1 \}}  \Bigl\{ a\beta_{\ell_1}\beta_{\ell_2} \\
&\quad+ (1-a) \frac{\delta_{\ell_1} - a\beta_{\ell_1}}{1-a} \frac{\delta_{\ell_2} - a\beta_{\ell_2}}{1-a} + B_{\ell_2} \delta_{\ell_1} + B_{\ell_1} \delta_{\ell_2} \Bigr\}^2\\
    \label{eq:calc_K2(Phi2Phi1)6}
    &\sim \sum_{\ell=1}^{L-1} \delta_{\ell}^2 +  \sum_{\ell=1}^{L-1}  a^2\beta_{\ell}^4 +  \sum_{\ell_1, \ell_2 \in \{1, \cdots, L-1 \}}  \Bigl\{ a\beta_{\ell_1}\beta_{\ell_2} + \left( a\beta_{\ell_1} \right) \left( a\beta_{\ell_2} \right)\Bigr\}^2\\
    \label{eq:calc_K2(Phi2Phi1)7}
    &= \sum_{\ell=1}^{L-1} \delta_{\ell}^2 +  \sum_{\ell=1}^{L-1}  a^2\beta_{\ell}^4 + \sum_{\ell_1, \ell_2 \in \{1, \cdots, L-1 \}} (a\beta_{\ell_1}\beta_{\ell_2})^2(1+a)^2\\
    \label{eq:calc_K2(Phi2Phi1)8}
    &\sim \sum_{\ell=1}^{L-1} \delta_{\ell}^2 +  \sum_{\ell=1}^{L-1}  a^2\beta_{\ell}^4 + \sum_{\ell_1, \ell_2 \in \{1, \cdots, L-1 \}} a^2\beta_{\ell_1}^2\beta_{\ell_2}^2.
\end{align}
By recursively applying lemma \ref{lem:zeta}\eqref{lem:bounded}, we can obtain that 
\begin{align}
  \label{eq:calc_K2(Phi2Phi1)9}
  K_2(\Phi_2(\Phi_1(v))) &\sim \sum_{\ell=1}^{L-1} \delta_{\ell}^2 + \sum_{\ell=1}^{L-1} a^2\beta_\ell^4 + \sum_{\ell_1, \ell_2 \in \{1, \cdots, L-1 \}} a^2\beta_{\ell_1}^2\beta_{\ell_2}^2.
\end{align}
Here, for all parameters $v$, 
\begin{align}
  \label{eq:calc_K2(Phi2Phi1)10}
  K_2(\Phi_2(\Phi_1(v))) \geq  \sum_{\ell=1}^{L-1} \delta_{\ell}^2 + \sum_{\ell=1}^{L-1} a^2\beta_\ell^4.
\end{align}
Also, since $\beta_{\ell_1}^2, \beta_{\ell_2}^2 \geq 0$ and the relation between the arithmetic mean and the geometric mean, 
\begin{align}
  \label{eq:calc_K2(Phi2Phi1)11}
  \sum_{\ell=1}^{L-1} \beta_\ell^4 + \sum_{\ell_1, \ell_2 \in \{1, \cdots, L-1 \}} \beta_{\ell_1}^2\beta_{\ell_2}^2 \ \leq \ \sum_{\ell=1}^{L-1} \beta_\ell^4 + \frac{1}{2}  \sum_{\ell_1, \ell_2 \in \{1, \cdots, L-1 \}}(\beta_{\ell_1}^4 + \beta_{\ell_2}^4).
\end{align}
Therefore, there exists a constant $k \geq 1$ that does not depend on $v$, and
\begin{align}
  \label{eq:calc_K2(Phi2Phi1)12}
  \sum_{\ell=1}^{L-1} \beta_\ell^4 + \sum_{\ell_1, \ell_2 \in \{1, \cdots, L-1 \}} \beta_{\ell_1}^2\beta_{\ell_2}^2 \leq k\sum_{\ell=1}^{L-1} \beta_\ell^4.
\end{align}    
By Eqs. \eqref{eq:calc_K2(Phi2Phi1)10}, \eqref{eq:calc_K2(Phi2Phi1)12},
\begin{align}
  \label{eq:calc_K2(Phi2Phi1)13}
  \sum_{\ell=1}^{L-1} \delta_{\ell}^2 + \sum_{\ell=1}^{L-1} a^2\beta_\ell^4 \leq K_2(\Phi_2(\Phi_1(v))) \leq k\Bigl( \sum_{\ell=1}^{L-1} \delta_{\ell}^2 + \sum_{\ell=1}^{L-1} a^2\beta_\ell^4 \Bigr).
\end{align}
Thus, let us define $K_3(v)$:
\begin{align}
  \label{eq:def:K3}
  K_3(v) = \sum_{\ell=1}^{L-1} \delta_{\ell}^2 + \sum_{\ell=1}^{L-1} a^2\beta_\ell^4,
\end{align}
from the lemma \ref{lem:zeta}\eqref{lem:enm:inequality} and Eq.\eqref{eq:calc_K2(Phi2Phi1)13},
\begin{align}
  \label{eq:K3 sim K2}
  K_3(v) \sim K_2(\Phi_2(\Phi_1(v))).
\end{align}
From Eq.\eqref{eq:K3 sim K2}, the main theorem can be derived by finding the maximum pole and its order of the zeta function determined from $K_3(v)$ and $\varphi(\Phi_2(\Phi_1(v)))$. 
If the prior distribution $\varphi(a)$ of mixture ratio $a$ is greater than $0$ and bounded, from the lemma \ref{lem:zeta}\eqref{lem:sum}\eqref{lem:product}, 
\begin{align}
  \lambda(K_3, \varphi) &= \sum_{\ell=1}^{L-1}\lambda(\delta_{\ell}^2) + \min \Bigl( \lambda(a^2),\ \sum_{\ell=1}^{L-1}\lambda(\beta_{\ell}^4) \Bigr)\\
  &= \frac{L-1}{2} + \min \Bigl(\frac{1}{2},\ \frac{L-1}{4} \Bigr).
\end{align}
Here, $\displaystyle \frac{1}{2} = \frac{L-1}{4}$ holds only when $L = 3$, and the equal sign does not hold in other cases, so the multiplicity is $m = 2$ only when $L = 3$ and $ m = 1 $ otherwise. Therefore, the theorem \ref{maintheorem}\eqref{maintheorem:enm1} was shown.

Furthermore, consider the case in the theorem \ref{maintheorem}\eqref{maintheorem:enm2}, that is, the case where the prior distribution of the mixing ratio $a$ follows the Dirichlet distribution with $\alpha > 0$ as the hyperparameter. 
Using that the prior distribution of $a$ is $\varphi(a) \propto a^{\alpha-1}(1-a)^{\alpha-1}$, and the prior distribution of other parameters is greater than $0$ and bounded,
\begin{align}
  \lambda(K_3, \varphi) &= \sum_{\ell=1}^{L-1}\lambda(\delta_{\ell}^2) + \min \Bigl(\lambda(a^2, a^{\alpha-1}), \ \sum_{\ell=1}^{L-1}\lambda(\beta_{\ell}^{4}) \Bigr)\\
        &= \frac{L-1}{2} + \min \Bigl( \frac{\alpha}{2},\ \frac{1}{4} \Bigr).
\end{align}
Therefore, we completed Theorem \ref{maintheorem}\eqref{maintheorem:enm2}.
\end{proof}

%% file: phase_transition.tex
In Bayesian statistics, if a prior distribution $\varphi(w ;\theta)$ has a hyperparameter $\theta$ and a posterior distribution for a sufficiently large $n$ changes drastically at $\theta = \theta_c$, then it is said that the posterior distribution has a \textbf{phase transition}, and $\theta_c$ is called a \textbf{critical point} \cite{watanabe2018mathematical}.

In the case of the theorem \ref{maintheorem}\eqref{maintheorem:enm2}, the prior distribution of the mixed ratio $a$ of the  multinomial mixture is assumed to be the Dirichlet distribution with the hyperparamater $\alpha$, and the asymptotic free energy $F_n(\alpha)$ is given by
\begin{align}
  \label{eq:phase_transition1}
  \bbE[F_n(\alpha) ] \simeq \begin{cases}
    nS + \frac{\alpha}{2}\log n & (\alpha < \frac{L-1}{2}) \\
    nS + \frac{3(L-1)}{4}\log n - \log \log n & (\alpha = \frac{L-1}{2})\\
    nS + \frac{3(L-1)}{4}\log n  & (\alpha > \frac{L-1}{2})\
  \end{cases}.
\end{align}

From Eq.\eqref{eq:phase_transition1}, at $\alpha_c = \frac{L-1}{2}$, $\bbE[F_n(\alpha)]$ is not differentiable, so $\alpha_c$ is the phase transition point. 
If there is a phase transition point, the support of the posterior distribution significantly changes between two phases which greatly affects the result of statistical inference.

%% file: conclusions.tex
In this paper, we derived the real log canonical threshold and multiplicity when the probability model and the prior are a multinomial mixture with two components and Dirichlet distribution respectively, and the true distribution is a multinomial distribution. The asymptotic behaviors of the free energy and the generalization error were clarified. One of future works is to find the RLCTs and multiplicities of the multinomial mixtures with the general number of components.